\documentclass{article}
\usepackage[utf8]{inputenc}
\input{preface}
\usepackage{natbib} 
\newcommand{\argmin}{\mathop{\mathrm{argmin}}\limits}

\title{
Coping with Mistreatment in Fair Algorithms}

\author{
Ankit Kulshrestha\\
Computer and Information Sciences\\
University of Delaware\\
Newark, DE\\
\texttt{akulshr@udel.edu}

\And

Ilya Safro\\
Computer and Information Sciences\\
University of Delaware\\
Newark, DE\\
\texttt{isafro@udel.edu}
}

\date{February 2021}

\begin{document}
\maketitle


\begin{abstract}
Machine learning actively impacts our everyday life in almost all endeavors and domains such as healthcare, finance, and energy. As our dependence on the machine learning increases, it is inevitable that these algorithms will be used to make decisions that will have a direct impact on the society spanning all resolutions from personal choices to world-wide policies. Hence, it is crucial to ensure that (un)intentional bias does not affect the machine learning algorithms especially when they are required to take decisions that may have unintended consequences. Algorithmic fairness techniques have found traction in the machine learning community and many methods and metrics have been proposed to ensure and evaluate fairness in algorithms and data collection. 

In this paper, we study the algorithmic fairness in a supervised learning setting and examine the effect of optimizing a classifier for the Equal Opportunity metric. We demonstrate that such a classifier has an increased false positive rate across sensitive groups and propose a conceptually simple method to mitigate this bias. We rigorously analyze the proposed method and evaluate it on several real world datasets demonstrating its efficacy.

\noindent {\bf Reproducibility:} All source code, and  experimental results are available at \url{https://anonymous.4open.science/r/b6c6653c-5f50-477c-bb13-7dfdeb39d4f4/}
\end{abstract}
\section{Introduction}
\label{sec:introduction}

Machine learning has permeated almost every sphere of human endeavours. The algorithms are being applied on both macro and micro scale in various degrees in diverse fields like space~\citep{kothari2020final}, finance~\citep{dixon2020machine}. Moreover, our dependence on machine learning algorithms to take decisions has created an urgent need to ensure that the algorithms are \emph{unbiased} when taking decisions, more so if that decision affects society.  There has been a steady growth of work in algorithmic fairness in the recent times, e.g.,  \citep{ahmadian_clustering_2019,ahmadian_fair_2020,chierichetti_fair_2017,dwork_decoupled_2017,pessach_algorithmic_2020}. The underlying assumption in most of the algorithmic fairness works is that there exists a unique sensitive attribute and it influences the decision made by the algorithm. However, in real world especially healthcare the problem is not so simple. 
For instance, in the healthcare data the number of sick people is usually significantly lower than that of the healthy people \cite{dua2014machine} in many data samples.  However, this statistical difference is not just influenced by that sensitive label (a  diagnosis code in the electronic health records) or any single other such as the gender of the patient but also on whether a patient has access to medical insurance, or whether a patient would want another surgery. In this situation, it becomes hard to not just quantify a fairness metric but also design algorithms that discover and optimize such domain specific metrics.

\begin{figure*}
    \centering
    \includegraphics[width=\textwidth]{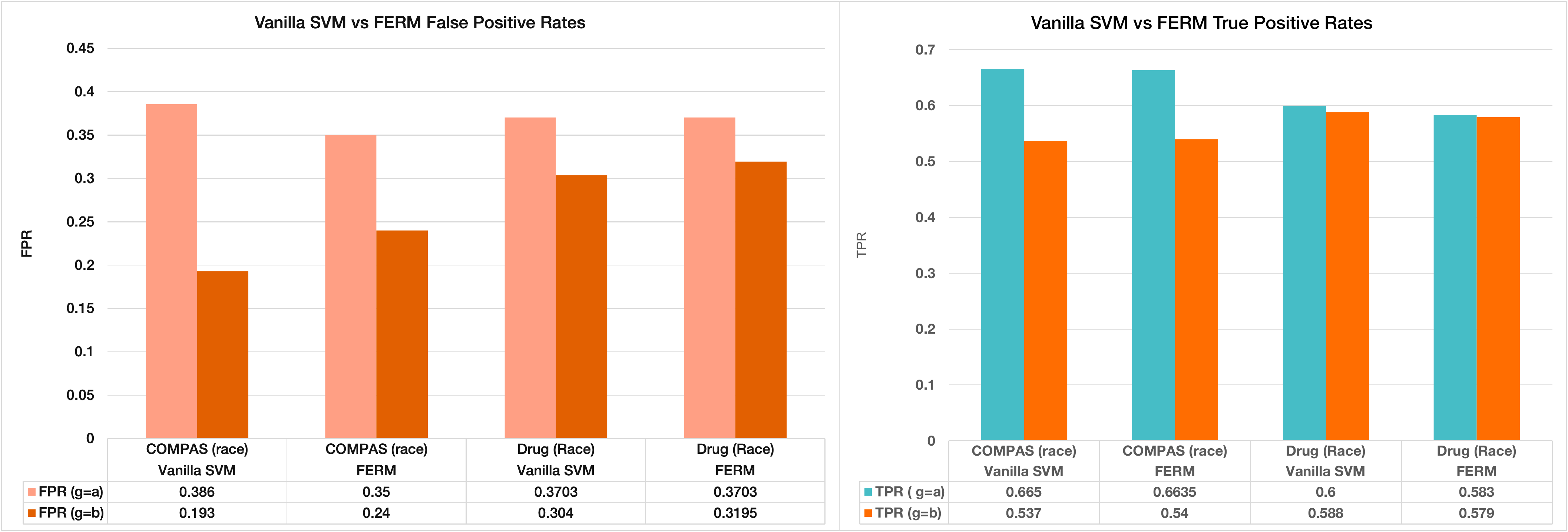}
    \caption{The True Positive Rates and False Positive Rates of Vanilla SVM vs Fair-Empirical Risk Minimization SVM. }
    \label{fig:fpr_tpr_sens}
\end{figure*}

For machine learning algorithms to be completely fair, they need to simultaneously satisfy the dual objectives of not preferring a statistical majority over a statistical minority to make a final decision and at the same time ensuring that the statistical majority is not unfairly treated to satisfy some constraint. In other words, a fair machine learning algorithm should be able to discriminate between different outcomes given a data sample without overly relying on a single sensitive attribute or a combination thereof. 

Many metrics such as Disparate Impact~\citep{Barocas2016BigDD,chierichetti_fair_2017,ahmadian_fair_2020, bera_fair_2019}, Demographic Parity~\citep{zafar_fairness_nodate, zliobaite2015relation, calders_demo_parity}, and Equal Odds/Equal Opportunity~\citep{hardt_equality_2016,donini_empirical_2020,zafar_fairness_2017} have been proposed to ensure the first fairness requirement but they haven't been studied to see if they also satisfy the second requirement as well. 

We focus on the supervised learning paradigm and assume that we have access to an underlying ``protected attribute" (we also refer to it as a ``sensitive attribute"). Two important fairness measures in this paradigm are the Equal Odds and Equal Opportunity measures~\citep{hardt_equality_2016}. The Equal Odds requires that a classifier has the same true positive rate (TPR) and false positive rate (FPR) across sensitive groups, i.e., 
 \begin{equation}
 |P[\hat{y} = 1 | s=0, y=1] - P[\hat{y} = 1 | s = 1, y=1]| \leq \epsilon
 \end{equation} 
 and 
 \begin{equation}
| P[\hat{y} = 1 | s=0, y=0] - P[\hat{y} = 0 | s = 1, y=0] | \leq \epsilon
\end{equation} 
where $\hat{y}$ is the output of the classifier, $y$ is the ground truth, $\epsilon$ is a user defined fairness criterion threshold and $s$ is an indicator variable defining the sensitive group. For our purposes, $s \in \{0,1\}$. Similarly, Equal Opportunity requires that \begin{equation}\label{eo:constr}
|P[\hat{y} = 1 | s=0, y=1] - P[\hat{y} = 1 | s = 1, y=1]| \leq \epsilon.
\end{equation} 
It has been shown in \citep{pleiss_fairness_2017} that the only way to satisfy the Equal Odds metric is to set $\epsilon = 0$. This requirement severely limits the areas in which the Equal Odds metric can be a useful measure for ensuring fairness. The Equal Opportunity can be seen as a relaxed version of Equalized Odds. When $\epsilon = 0$, Equal Opportunity requires that the true positive rates across sensitive groups are same. Figure~\ref{fig:fpr_tpr_sens} shows the true positive and false positive rates of an SVM model trained without fairness constraint and another SVM trained with Equal Opportunity (EO) constrain across the sensitive groups on two datasets - COMPAS recidivism dataset~\citep{compasdata} released by ProPublica and the Drug dataset~\citep{fehrman2017factor}. We can see that in both datasets, the false positive rates increase across the sensitive groups. Moreover, the increase in TPR of a statistical minority group comes at the expense of a statistical majority group. We name this phenomenon as \emph{disparate mistreatment} (different from~\cite{zafar_fairness_nodate}).  A classifier trying to ensure equal opportunity can make myopic decisions by inadvertently associating the sensitive group as the most important feature for final prediction. For instance, in the healthcare domain the data is often extremely imbalanced \citep{razzaghi2019predictive,razzaghi2015fast} due to such reasons as much smaller number of sick people than healthy, unequal access to medical insurance, and incorrect sampling among the people seeking for medical help. A trained classifier with Equal Opportunity metric may falsely output more positive responses on the group forming a minority in the overall disease rate. This bias would then violate our requirement on fairness.

\subsection{Our Contributions}
In this paper, we note the existence of bias exhibited by classifiers optimizing EO metric \emph{against} majority groups belonging to a single protected attribute. We make the following contributions with our work:

\begin{itemize}
    \item We concretely define the disparate mistreatment phenomenon presented above and propose a conceptually simple criteria to prevent the disparate mistreatment. 
    
    \item We rigorously show that minimizing the optimization objective subject to our constraint can generalize to unseen samples. 
    
    \item We quantify the performance of classifier in terms of the DFPR metric (difference of False Positive Rates) and show that a classifier trained with our approach exhibits a significantly less DFPR when compared against other baselines.
\end{itemize}

\section{Related Work}
\label{sec:related_work}

Work on algorithmic fairness has progressed on different areas of machine learning and different algorithms have been proposed in both supervised and unsupervised settings. In this paper we limit ourselves to work concerning statistical or group fairness. For an excellent introduction to algorithmic fairness in general, please consider~\citep{pessach_algorithmic_2020} and the references therein.

In the supervised learning paradigm, the proposed algorithms achieve fairness by either data-preprocessing, optimizing a model satisfying a particular criteria or fine tuning an unoptimal classifier to meet fairness criteria. For instance, ~\citep{hardt_equality_2016} propose the Equal Odds and Equal Opportunity fairness metric and first train an unoptimal (in the fairness sense) classifier. Then the authors fine tune the classifier keeping in mind the protected attribute of the dataset. 

In a different  approach~\citep{donini_empirical_2020} propose the FERM framework to optimize the Equal Opportunity metric by integrating the fairness constraint in an Empirical Risk Management (ERM) framework. ~\citep{zafar_fairness_2017} propose a framework to mitigate disparate mistreatment by measuring the distance of users feature vectors that are misclassified with the decision boundary. The authors then propose to mitigate disparate impact(DI)~\citep{zafar_fairness_nodate} by measuring the distance between a users feature vector and the decision boundary. We note that our definition of disparate mistreatment is different from~\citep{zafar_fairness_2017} and also note that their method assumes linear separability of data while we do not make any such assumptions about our data. Our work is more closely related to~\cite{donini_empirical_2020} since we also use the ERM framework to integrate fairness in the algorithm. 

Another direction to ensure fairness in supervised setting is proposed by~\citep{zemel_learning_nodate} in which a \emph{new} data representation is derived from the original existing data in a way that removes the protected attribute from participating directly into the decision process. An ensemble learning approach is proposed by~\citep{dwork_decoupled_2017} in which decoupled classifiers are trained and    each classifier is fine tuned jointly by minimizing a joint loss. 

Fairness and calibration are closely related. ~\citep{pleiss_fairness_2017} define calibration and rigorously show that a classifier cannot simultaneously satisfy Equalized Odds and output calibrated responses. Moreover, it has been shown that the only way to ensure Equalized Odds is to ensure that both true positive rates and false positive rates across two sensitive groups are exactly the same.

In the unsupervised learning domain, fairness has been considered within the framework of clustering. Following the work of~\citep{chierichetti_fair_2017} on fairlets, different methods have been proposed~\citep{ahmadian_clustering_2019, ahmadian_fair_2020, chhabra_fair_2020, backurs_scalable_2019} that either solve the clustering problem in a heirarchical agglomerative setting or ensuring fairness in a k-means clustering setting. The central idea in all works in this vein is to prevent an over-representation of a particular sensitive group in a cluster.

\section{Our Approach}
\label{sec:method}

We begin with the notation. Let $\mathcal{D} = \{(x_{i}, y_{i})\}^{n}_{i=1}$ be a labelled dataset of $n$ samples, where $x_i$ and $y_i$ represent the data point and label, respectively. Let $\cal{F}$ be the space of all functions that minimize a prescribed loss function $l(f(x_{i}), y_{i}),~ \forall f \in \cal{F}$. We define $f_{a}$ to be a classifier that optimizes accuracy and $f_{eo}$ to be a classifier that minimizes $l(\cdot, \cdot)$ taking into account fairness using the Equal Oportunity (EO) constraint. Further, we constrain the target classes $\mathcal{Y} \in \{-1, +1\}$. For any classifier, let
\begin{equation}
\mathbb{P}_{g}(\hat{y}, y) = P[\hat{y} | s = g, y]
\end{equation}
be the probability of predicting $\hat{y}$ for sensitive group $g$ and target $y \in \mathcal{Y}$. We further assume that there are only two classes of protected attribute hence $g \in \{a, b\}$.
We are now ready to define the disparate mistreatment phenomenon. 

\theoremstyle{definition}
\begin{definition} [Disparate Mistreatment]
\textit{A classifier $f$ is said to exhibit disparate mistreatment if and only if for groups $a, b$ an increase in $\mathbb{P}_{b}(1, 1)$ directly corresponds to a decrease in $\mathbb{P}_{a}(1,1)$ and vice versa.}
\label{def:dmt}
\end{definition}
It follows from Definition~\ref{def:dmt} that any classifier $f$ exhibiting disparate mistreatment will have an increased rate of false negatives for one statistical group and at the same time also result in an increased false positive rate in the other. For any $f \in \mathcal{F}$, optimizing for maximizing accuracy will return $f_{a}$ which tries to find the decision boundary that differentiates between positive and negative samples conditioned on the outcome alone. 

However, optimizing EO results in $f_{eo}$ which constrains the classifer to find a supoptimal decision boundary so that true positives increase across the different sensitive groups. Geometrically, this can be seen as pushing the clusters of positively labeled samples in either group closer to each other by minimizing the distance between their respective means. This optimization method however completely ignores the effect on the cluster of negatively labeled samples in either sensitive group. We hypothesize that this is the root cause of $f_{eo}$ exhibiting disparate mistreatment. 

Our approach is to introduce a \emph{minimum separation parameter} into the Empirical Risk Minimization (ERM) framework. We build on the work of ~\citep{donini_empirical_2020} due its attractive guarantees on upper bounds of the empirical risk.

\theoremstyle{definition}
\begin{definition}
\emph{Let 
\[
R^{-, g}(f) = \mathbb{E}([l(f(x_{i}), y_{i}) \mid y=-1, s=g])
\]
be the risk on negatively labeled samples for a particular sensitive group in a dataset. Then $f$ is $\rho$-minimum separated if 
\begin{equation}|R^{-, a}(f) - R^{-, b}(f)| \geq \rho.
\end{equation}}
\label{def:rho_sep_fmt}
\end{definition}

Based on Definition~\ref{def:rho_sep_fmt}, we can formulate two constrained ERM problems. In the case where we want to obtain $f$ such that the decision boundary respects a minimal separation between negative label clusters we optimize:

\begin{equation} \label{eq:rho_sep_opt}
\begin{aligned}
     \underset{f\in  \mathcal{F}}{\operatorname{\argmin}} & & R(f)\\
     \text{such that} & & |R^{-, a}(f) - R^{-, b}(f)| \geq \rho
\end{aligned}
\end{equation}

We call Problem~\ref{eq:rho_sep_opt} as \emph{Equal Treatment }(ET). In this case, we do not constrain the optimal function to ensure equal true positive rates across sensitive groups. Instead, fairness is achieved by ensuring that negative samples from any sensitive group do not get classified as positive up to the tolerance level $\rho$. We additionally require that $0 < \rho \leq 1$. However, there may be situations where we would want both EO and ET constraints to be satisfied. In those cases we optimize:

\begin{equation} \label{eq:eo_et_opt}
\begin{aligned}
        \underset{f\in  \mathcal{F}}{\operatorname{\argmin}} & & R(f)\\
        \text{such that} & & |R^{-, a}(f) - R^{-, b}(f)| \geq \rho \\ 
        & & |R^{+, a}(f) - R^{+, b}(f)| \leq \epsilon
\end{aligned}
\end{equation}

We have defined the optimization problems~\ref{eq:eo_et_opt} and ~\ref{eq:rho_sep_opt} in a general manner. We now analyze the  bounds on the performance of our algorithm. First, we define $R(f)$ to be the \emph{true risk} of the unknown underlying data distribution $P$. Further, 
\[
\hat{R}(f) = \frac{1}{n}\sum_{i=1}^{n} l_{c}(f(x_{i}, y_{i}))
\]
is the \emph{empirical risk} associated with samples $(x_{i}, y_{i})$ drawn i.i.d from $\mathcal{D}^{n}$. We call $f^{*}$ to be the Bayes optimal classifier minimizing true risk and $\hat{f}$ to be optimal classifier minimizing the empirical risk. Furthermore, we constrain the loss function $l_{c}$ to be convex. 


\begin{lemma}\label{lemm:erm_risk}
If $f \in \mathcal{F}$ be a function minimizing empirical risk subject to constraints in Problem~\ref{eq:rho_sep_opt} and the loss $l_{c}$ is convex. If $C(\mathcal{F}, \delta, n^{-})$ is the upper bound function on empirical risk,  then with probability of at least $1 - 4\delta$, 
\begin{equation}
\begin{aligned}
|R^{-, a}(f^{*}) - R^{-, b}(f^{*})| < & |R^{-, a}(\hat{f}) - R^{-, b}(\hat{f})| + \\
 & \rho +  2C(\mathcal{F}, \delta, n^{-}).
\end{aligned}
\end{equation}    
\end{lemma}

\begin{proof}

For a given upper bound function $C(\mathcal{F}, \delta, n^{-})$, the true risk and empirical risk are related as:
\begin{equation} \label{eq:risk-assoc}
 R(\hat{f}) < \hat{R}(f) + C(\mathcal{F}, \delta, n)
\end{equation}
with probability greater or equal $(1-\delta)$. 
Using Equation~\ref{eq:risk-assoc}, we can that with probability  of atleast $1 - 2\delta$, where $\delta$ is the confidence parameter that 
\begin{equation}
|R^{-,a}(\hat{f}) - R^{-, b}(\hat{f})| < |\hat{R}^{-, a}(f) - \hat{R}^{-, b}(f)| +  C(\mathcal{F}, \delta, n).
\end{equation}
But, we have defined $|\hat{R}^{-, a}(f) - \hat{R}^{-, b}(f)| \geq \rho$. Then it follows that,

\begin{equation}\label{eq:interim}
\begin{aligned}
 |R^{-,a}(\hat{f}) - R^{-, b}(\hat{f})| < & |\hat{R}^{-, a}(f) - \hat{R}^{-, b}(f)| +\\ &  C(\mathcal{F}, \delta, n^{-}) + \rho.   
\end{aligned}
\end{equation}
Essentially, Equation~\ref{eq:interim} means that the true risk of a classifier that minimizes ERM is upper bounded by the cost function and the minimal separation parameter $\rho$. However, we are interested in bounding the risk of the optimal classifier on the training data (observed samples) in terms of the Bayes optimal classifier $f^{*}$, so the following relation is applied 
\begin{equation}\label{eq:optimal-risk}
\hat{R}(f^{*}) < R(f^{*}) +C(\mathcal{F}, \delta, n)
\end{equation}
with probability greater or equal $(1-\delta)$. 
Using Equation~\ref{eq:optimal-risk}  and noticing that the right hand side of Equation~\ref{eq:interim} is valid $\forall f \in \mathcal{F}$, we can state that with probability  at least $1 - 4\delta$

\begin{equation}\label{eq:proof-result}
\begin{split}
    |R^{-, a}(\hat{f}) - R^{-, b}(\hat{f})| <  |R^{-, a}(f^{*}) - R^{-, b}(f^{*})| +\\ \rho +
    2*C(\mathcal{F}, \delta, n^{-}).
\end{split}
\end{equation}
\end{proof}

The result of Lemma~\ref{lemm:erm_risk} upper bounds the empirical risk of the optimal and can guarantee that with a small penalty of $\rho$ the empirical risk will be close to the true risk on the underlying distribution. However, the degree of closeness will depend on how fast $C(\mathcal{F}, \delta, n)$ converges to a uniform value. If $\mathcal{F}$ is finite we can use Union bound to derive the upper bound function as:

\begin{equation} \label{eq:uniform-bound}
    C(\mathcal{F}, \delta, n) = \sqrt{\frac{log(\mathcal{|F|}) + log(1/\delta)}{2n}}
\end{equation}

On the other hand, if $\mathcal{F}$ is infinite then we can make use of Rademacher bound to find an upper bound for the empirical risk. The key point to note is that $C(\mathcal{F}, \delta, n)$ decays to $0$ as $n \rightarrow \infty$ which implies that the empirical risk can approximate the true risk within a small penalty of $\rho$.

\subsection{Non-Linear Optimization with Kernels}

In this section we specify our problem when the underlying space of functions is the Reproducing Kernel Hilbert Space (RKHS). For any function $f$ in RKHS we parametrize it as $f(x) = \langle \textbf{w}, \phi(\textbf{x})\rangle$ where $w \in \mathbb{H}$ and $\phi(\textbf{x})$ is a feature map in high dimensional space. For our purpose, we assume $\phi(\textbf{x})$ is a \emph{non-linear} mapping in high dimensional space, i.e., $\phi: \mathcal{X} \rightarrow \mathbb{H}$. Since our underlying function space is RKHS, we can define a positive semidefinite kernel matrix $K: \mathcal{X}\times \mathcal{X} \rightarrow \mathbb{R}$, such that $K(x, x') = \langle \phi(x), \phi(x') \rangle$. 

Following~\citep{donini_empirical_2020}, we define  $\textbf{u}_{g}$ to be the barycenter of a negatively labeled cluster of points belonging to a sensitive group $g$: 
\begin{equation} \label{eq:diff_of_barycenters}
    \textbf{u}_{g} = \frac{1}{n^{-, g}} \sum_{i \in \mathcal{Y}^{-}, g=g}^{n} \phi(x_{i})
\end{equation}

In general, we optimize the following problem

\begin{equation}
\begin{aligned}
    \min_{\textbf{w}} & & \sum_{i = 1}^{m} l(\langle \textbf{w}, \phi(x_{i})\rangle, y_{i}) + ||\textbf{w}||_{2}^{2}\\
    \text{such that} & &   \langle \textbf{w}, \Bar{\textbf{u}} \rangle \geq \rho
\end{aligned}
\label{eq:main_optimization}
\end{equation}

In Equation~\ref{eq:main_optimization}, $\Bar{\textbf{u}} = \textbf{u}_{a}- \textbf{u}_{b}$  . We further specify this problem to a soft-margin SVM by optimizing

\begin{equation}\label{eq:soft-margin-svm}
    \begin{aligned}
        \min_{\textbf{w}} & &\textbf{w}^{T}\textbf{w} + C\sum_{i=1}^{n} \zeta_{i} \\
        \text{such that} & & \zeta_{i} \geq 0 \\
        & &  y_{i}(1-\phi(x_{i})^{T}\textbf{w}) \geq 1 - \zeta_{i} \\ 
        & & \langle \textbf{w}, \Bar{\textbf{u}} \rangle \geq \rho
    \end{aligned}
\end{equation}

To simplify solving Equation~\ref{eq:main_optimization}, we use the Representer theorem~\citep{kimeldorf1971some}  to note that $\textbf{w} = \sum_{i=1}^{n} \alpha_{i}\phi(x_{i})$, then 
\begin{equation}
\langle \textbf{w}, \phi(\textbf{x})\rangle = \sum_{i=1}^{n} \alpha_{i} \langle \phi(x), \phi(x') \rangle,
\end{equation}
which can be simply expressed as 
\begin{equation}
\langle \textbf{w}, \phi(\textbf{x})\rangle = \sum_{i=1}^{n} \alpha_{i} k(x_{i}, x'_{i}),
\end{equation}
where $k(x_{i}, x'_{i})$ is an element of the gram matrix $K$. We also write, $||\textbf{w}||_{2}^{2} = \sum_{i,j=1}^{n} \alpha_{i}\alpha_{j} K(x_{i}, x_{j})$. The optimization problem can be restated as:

\begin{equation} \label{eq:opt-prob-secondary}
    \begin{aligned}
        \min_{\alpha} & & \sum_{i=1}^{n} l(\alpha_{i} * K(x_{i}, x'_{i}), y_{i}) + \sum_{i,j=1}^{n} \alpha_{i}\alpha_{j} K(x_{i}, x_{j}) \\ 
        \text{such that} & & \sum_{i=1}^{n} \alpha_{i}\left[\frac{1}{n^{-, a}}\sum_{y \in \mathcal{Y}^{-}}^{n} K_{iy} - \frac{1}{n^{-, b}}\sum_{y \in \mathcal{Y}^{-}}^{n} K_{iy}\right] \geq \rho
    \end{aligned}
\end{equation}

\textit{Geometric Intuitions}: We stated earlier that $\textbf{w} \in \mathbb{H}$ where $\mathbb{H}$ is the Hilbert space. Since Hilbert space is a complete vector space, we can interpret any $\langle \textbf{x}, \textbf{y} \rangle$ as a distance between two infinite dimensional vectors embedded in that space. Moreover, since the underlying function space is a RKHS, any functional will be upper bounded by a finite quantity. If we solve Equation~\ref{eq:main_optimization} without the constraint, then we simply are choosing a weight vector which is closest to points $\phi(x_{i})$ that minimize the loss. However, our added constraint forces the algorithm to search for a weight vector that minimizes the loss \emph{while} maintaining a distance of at-least $\rho$ between the center of masses of points that have negative labels associated with them.

\section{Experimental results}
\label{sec:experiments}

\begin{table}[t]
    \centering
    \resizebox{\columnwidth}{!}{
    \begin{tabular}{l c c c}
    \toprule
        Dataset  & \# Data points & \# Features & Sensitive Attribute  \\
    \midrule
        Adult & $32561$/$16281$ & $12$ & Gender  \\
        COMPAS & $6171$ & $10$ & Race \\ 
        Drug & $1889$ & $11$ & Race  \\ 
        Arrhythmia & $452$ & $279$ & Gender \\
    \bottomrule
    \end{tabular}
    }
    \caption{Statistics of Real World Datasets.}
    \label{tab:dataset_stats}
\end{table}
We now describe our experimental results on several real-world datasets and show the effectiveness of our approach in reducing the FPR across the sensitive groups. Table~\ref{tab:dataset_stats} summarizes the datasets.

\noindent \textbf{Adult} is a dataset from UCI repository \citep{Dua:2019} containing 14 features with 48842 data points about different aspects of an individual (such as age, gender, and marital status) to determine if that person earns more than \$50,000 a year. For this dataset we consider \emph{gender} to be the sensitive attribute. 

\noindent \textbf{COMPAS} Correctional Offender Management Profiling for Alternative Sanctions (COMPAS) algorithm~\citep{compasdata} used by Broward County Prison to assess the re-arrest rate of different individuals was shown to be biased against black defendants. We used the dataset \citep{compasdata} provided by Propublica and specifically focused on the violent-recidivism subset. The dataset contains 6,171 data points with 12 features. We chose \emph{race} as the sensitive feature.

\noindent \textbf{Drug} is a dataset provided by UCI \citep{Dua:2019} containing information about 1889 individuals. Each individual has 12 features associated with them including personality type, ethnicity and whether they have used a drug in the past. There are 16 possible classification problems concerning use of different drugs. In this paper we select the use of cocaine as our target variable and discretize it into ``Never used'' and ``Used'' classes. We further discretize the sensitive attribute \emph{ethnicity} into ``White" and ``Other" classes.

\noindent \textbf{Arrhythmia} is a dataset containing 452 samples and 279 features to detect the absence or presence of cardiac arrhythmia in a patient. We also turn this into a binary classification problem by considering ``Normal" vs ``Others" as the outcome variable. We choose \emph{gender} as the protected attribute. Since this data contains a lot of missing values we pre-processed the data by dropping feature with number of unknowns greater than a threshold and substituting the rest of the unknowns with the median value of the feature. This dataset is also provided by the UCI~\citep{Dua:2019} repository. 

We use the RBF kernel 
\[
K(\textbf{x},\textbf{y}) = e^{\gamma ||\textbf{x}-\textbf{y}||^{2}}
\]
for all experiments. All real valued features are scaled to have a zero mean and unit standard deviation. We further vectorize all categorical variables as well. In datasets where the test/train split is not provided we search for the optimal hyperparameters ($C, \gamma$) using 5-fold cross validation for each given value of $\rho$ and report the best result.  For baseline, we use the code provided by~\citep{donini_empirical_2020} with our pre-processed data.
In all tables containing the results, our method is labeled by SVM-MT, the original FERM and SVM by FERM and SVM, respectively. 

\begin{table*}[]
    \centering
    \setlength{\tabcolsep}{.3em}
    \begin{tabular}{lcccccccc}
    \toprule
     Method   & 
    \multicolumn{2}{c}{Adult} & 
    \multicolumn{2}{c}{COMPAS} & 
    \multicolumn{2}{c}{Drug} & 
    \multicolumn{2}{c}{Arrhythmia}\\
    \cmidrule(lr){2-3}
    \cmidrule(lr){4-5}
    \cmidrule(lr){6-7}
    \cmidrule(lr){8-9}
    {} & 
    Precision & DFPR &
    Precision & DFPR &
    Precision & DFPR &
    Precision & DFPR \\
    \midrule
    SVM  & 0.74 & 0.086 & $0.668 \pm 0.009$ & $\textbf{0.217}\pm 0.058$ & $0.672 \pm 0.027$ & $0.098 \pm 0.066$ & $0.747 \pm 0.061$ & $\textit{0.158} \pm 0.103$ \\ 
    SVM-MT & 0.73 &  \textbf{0.020} & $0.664 \pm 0.023$ & $\textit{0.225} \pm 0.058$ & $0.625 \pm 0.017$ & $\textbf{0.033} \pm 0.026$ & $0.703 \pm 0.072$ &  $\textbf{0.056} \pm 0.04$ \\ 
    FERM  & 0.71 & \textit{0.067} & $0.667 \pm 0.015$ & $0.25 \pm 0.053$ & $0.627 \pm 0.023$& $\textit{0.069} \pm 0.078$ & $0.742 \pm 0.048$ & $0.21 \pm 0.124$ \\
    \end{tabular}
    \caption{A comparison of results on four datasets with baseline FERM method and optimizing for minimal separation. Best results are shown in bold and second best are in italics.}
    \label{tab:results_main}
\end{table*}


In order to measure the performance of the classifier after optimization we define the \emph{DFPR} metric to be the absolute difference between the FPR across the sensitive groups. Our search for optimal $\rho$ was performed by first optimizing the SVM specific hyperparameters via cross fold validation. When we do not explicitly balance the dataset, then we sweep the values of $\rho$ in $[0.1,\dots, 0.5]$ in steps of $0.1$ to determine the optimal value that minimizes the DFPR. 

In our experiments, we did not notice a significant decrease of DFPR value for $\rho > 0.5$. We explain this as follows: when $\rho$ is higher then a maximum margin classifier would effectively optimize for accuracy. Table~\ref{tab:results_main} shows the results of our experiments on four real world datasets. \emph{We observe that our method consistently outperforms the baseline FERM~\citep{donini_empirical_2020} method achieving significantly lower DFPR on almost all datasets without sacrificing too much precision}.

\begin{table}[h]
    \centering
    \begin{tabular}{lcccc}
    \toprule
        Model & Adult & COMPAS & Drug & Arrhythmia \\
    \midrule
        SVM &  $1 \pm 1.5$ & $3 \pm 1$ & $1\pm 1$ & $0.2 \pm 0.5$ \\ 
        SVM-MT & $7 \pm 1 $ & $121\pm 1$ & $10\pm 1$& $1\pm0.6$ \\ 
        FERM & $10 \pm 1$ & $136\pm 1$& $16\pm 1$& $2 \pm	1$ \\ 
    \bottomrule
    \end{tabular}
    \caption{Model optimization times on different datasets. The results are an average of 5 different runs. All times except for Adult dataset are in seconds.  Adult dataset time is measured in minutes.}
    \label{tab:timing_results}
\end{table}
For the completeness of comparison we also trained SVM models  without any constraints. As expected, the unconstrained SVM has a higher accuracy, but also a higher rate of FPR than any other models. In particular, the FERM model has a comparable precision with the baseline SVM but has a much higher DFPR rate on the Arrhytmia dataset, while on the other hand our method sacrifices some precision to achieve the lowest DFPR rate of $0.05$ with $\rho=0.1$. Table~\ref{tab:timing_results} shows the average optimization time over 5 different runs for all three models considered in our study. As expected, an unconstrained baseline SVM is significantly faster than the other two methods in finding the decision boundary. However, compared against the FERM model, our approach is consistently faster. In the case of Adult dataset optimizing both FERM and SVM-MT takes considerably longer than the baseline SVM. This indicates that for larger datasets a more scalable algorithm needs to be designed.

\begin{figure}
    \centering
    \includegraphics[width=\linewidth]{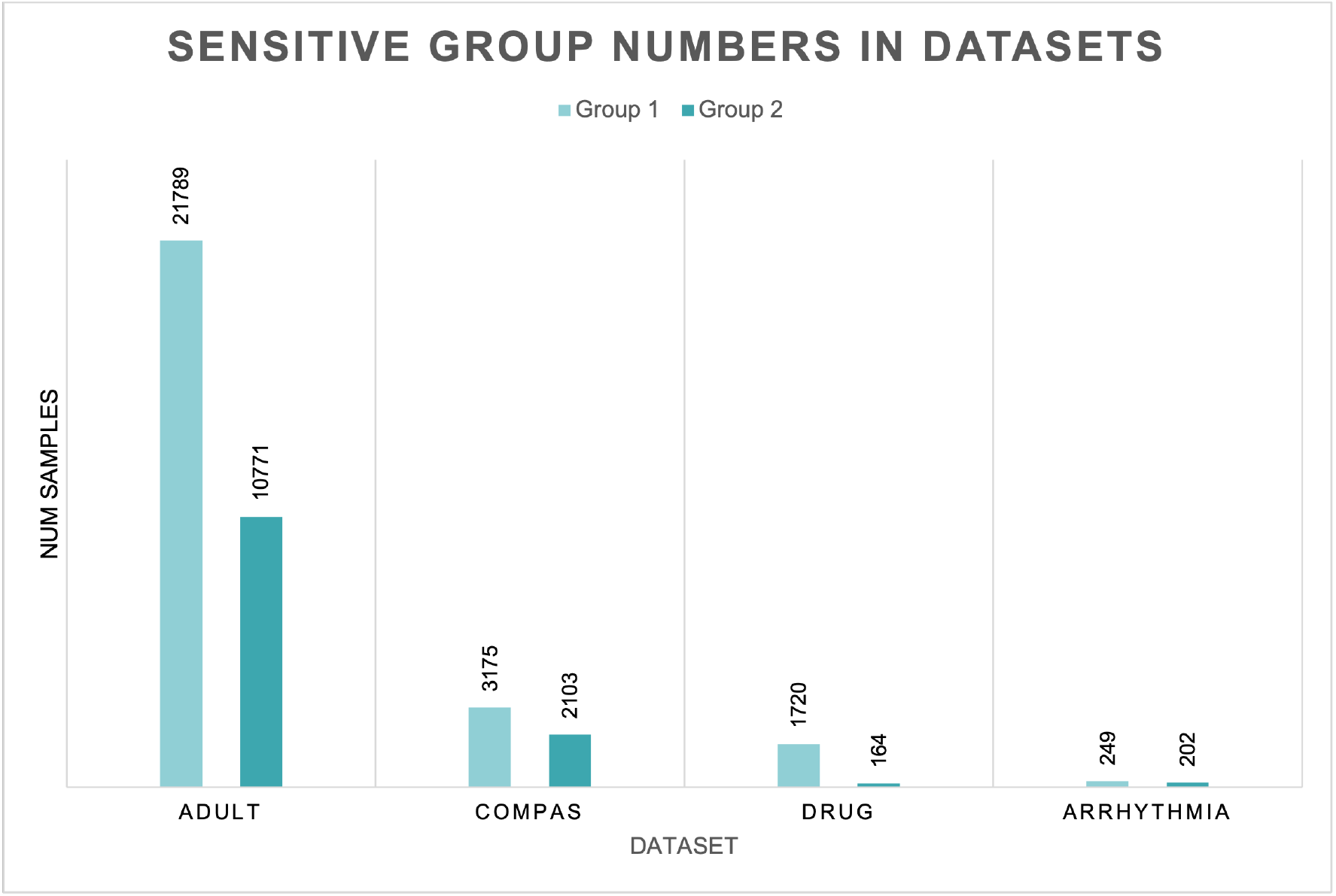}
    \caption{Statistics about sensitive groups in real world datasets.}
    \label{fig:dataset_stats}
\end{figure}

\begin{table}[ht]
    \centering
    \resizebox{\columnwidth}{!}{
    \begin{tabular}{lcccc}
    \toprule
        Model &  \multicolumn{2}{c}{Drug} & \multicolumn{2}{c}{Arrhythmia} \\
        \cmidrule(lr){2-3}
        \cmidrule(lr){4-5}
         &  Precision & DFPR & Precision & DFPR\\
    \midrule
        SVM  &  $0.672 \pm 0.028$ & $0.12 \pm 0.087$ & $0.744 \pm 0.05$  & $0.251 \pm 0.135$ \\ 
        SVM-MT &  $0.661 \pm 0.032$ & $0.074 \pm 0.04 $& $0.703 \pm 0.03$ & $0.09 \pm 0.06$\\ 
        FERM  & $0.648\pm 0.017$& $0.102 \pm 0.05$& $0.735 \pm 0.03$ & $0.121 \pm 0.07$\\ 
    \bottomrule
    \end{tabular}
    }
    \caption{Results of experiments with undersampling the majority label class on Drug and Arrhythmia datasets. The results are an average of 5 different runs.}
    \label{tab:sampling_expts}
\end{table}

\subsection{Handling Imbalanced Classes}

The real world datasets have a class imbalance since data collection is not always uniform. For instance, in the Drug dataset there are 784 people belonging to sensitive group (``Other") that are labeled as a positive class. By contrast, only 63 people belonging to the other sensitive group(``White") are positively labeled. The same trend is also observed in other datasets. In order to evaluate the algorithms more fairly, we wanted to remove the imbalance in the class labels. Hence, we undersampled the numerically superior class and ran the optimization against all three methods i.e. baseline SVM, SVM-MT(our method) and FERM. Table~\ref{tab:sampling_expts} shows the results of our experiments. All the results are the average of 5 independent runs. We found $\rho = 0.1$ to be a stable choice for all the experiments. 

In the case of Drug dataset, FERM model has a significantly reduced precision, while our method has almost comparable precision to the baseline. Moreover, our DFPR is the lowest amongst all three methods. Given that the drug dataset exhibits a strong class imbalance and FERM optimizes to minimize the TPR, undersampling the majority  decreases the precision. In the Arrhythmia dataset we also note that our method achieves the best DFPR. However, for the FERM model the precision drops less due to the fact that difference between positively labeled samples is not as pronounced (160 in one sensitive group and 85 in the other). We also observe that our method does exhibit a lower precision on the Arrhythmia dataset which we interpret to be because of the low number of samples in the data and our objective to keep the negatively labeled samples separated by atleast $\rho$. 


\section{Conclusion}
\label{sec:conclusion}

In this paper we have studied the effects of optimizing a classifier under the Equal Opportunity constraint. We showed that a classifier trying to minimize the difference between true positive rates of sensitive groups causes an increase in the false positive rates within each group. Furthermore, the increase in true positive rate of the statistical minority is achieved by the decrease in the true positive rate of the statistical majority. While this may be a valid method to ensure fairness with a single protected attribute, we reasoned that it would be counterproductive if the sensitive attribute could not be precisely defined. To mitigate this mistreatment we proposed a minimum separation parameter and showed that our approach is successful in decreasing the false positive rates across the sensitive groups without sacrificing too much accuracy. 

Our work opens up a discussion about how far the assumption of a unique sensitive attribute in data is valid. A fruitful direction of work can be to examine the existing fairness metrics when the sensitive attribute is not precise. Another direction that opens up is to define metrics that are robust to imprecise sensitive attributes and create algorithms that make minimal assumptions about the underlying data. Another important direction is related to introducing the Equal Opportunity constraint into the recent scalable accelerators of nonlinear support vector machines such as  \citep{wen2018thundersvm,sadrfaridpour2019engineering} which are expected to provide a notable scalability improvement to the entire optimization process.

\noindent {\bf Reproducibility:} All source code, and  experimental results are available at \url{https://anonymous.4open.science/r/b6c6653c-5f50-477c-bb13-7dfdeb39d4f4/}

\bibliography{fairness_ml}

\end{document}